\def\nips{0} 

\pdfoutput=1

\ifnum\nips=1

\documentclass{article}

\usepackage{neurips_2019}

\usepackage[utf8]{inputenc} 
\usepackage[T1]{fontenc}    

\usepackage{booktabs}       %
\usepackage{amsfonts}       
\usepackage{nicefrac}       
\usepackage{microtype}      

\else

\documentclass[11pt]{article}
\usepackage{fullpage}

\fi

\usepackage{hyperref}
\usepackage{url}

\usepackage{mathtools}
\usepackage{times}
\usepackage{graphicx,color}
\usepackage{array,float}
\usepackage[usenames,dvipsnames]{xcolor}
\usepackage{amstext,amssymb,amsmath}
\usepackage{amsthm}
\usepackage{verbatim}
\usepackage{bm}
\usepackage{paralist}
\usepackage{ulem}
\ifnum\nips=0
\usepackage[numbers]{natbib}

\fi

\usepackage{todonotes}
\usepackage[noend]{algorithmic}
\usepackage{algorithm}

\ifnum\nips=1
\usepackage{titlesec}
\titlespacing{\section}{0pt}{0.25\baselineskip}{-0.15\baselineskip}
\titlespacing{\subsection}{0pt}{0.50\baselineskip}{-0.15\baselineskip}
\usepackage{enumitem}
\setitemize{noitemsep,topsep=0pt,parsep=0pt,partopsep=0pt}
\fi

\newtheorem{lem}{Lemma}[section]

\newcommand{\lbeta}{\ell_{\beta}}

\newtheorem{thm}[lem]{Theorem}

\newtheorem{defn}[lem]{Definition}
\newtheorem{fact}[lem]{Fact}
\newtheorem{assumption}[lem]{Assumption}

\newcommand{\pr}[2]{\underset{#1}{\mathbb{P}}\left[ #2 \right]}
\newcommand{\ex}[2]{\underset{#1}{\mathbb{E}}\left[ #2 \right]}

\renewcommand{\paragraph}[1]{\vspace{3pt}\noindent\textbf{#1}}

\newcommand{\Aobj}{\mathcal{A}_{\sf ObjP}}

\newcommand{\iden}{\mathbb{I}}

\newcommand{\norm}[1]{\|#1\|}

\newcommand{\eps}{\epsilon}
\newcommand{\A}{\mathcal{A}}
\newcommand{\D}{\mathcal{D}}

\newcommand{\Z}{\mathcal{Z}}
\newcommand{\W}{\mathcal{W}}

\newcommand{\er}{\Delta\mathcal{L}}
\newcommand{\her}{\Delta\widehat{\mathcal{L}}}

\newcommand{\cI}{\mathcal{I}}
\newcommand{\V}{\mathcal{V}}

\newcommand{\cL}{\mathcal{L}}
\newcommand{\hL}{\widehat{\mathcal{L}}}

\newcommand{\hF}{\widehat{\mathcal{F}}}
\newcommand{\cF}{\mathcal{F}}

\newcommand{\hbw}{\widehat{\mathbf{w}}}
\newcommand{\hnabla}{\widehat{\nabla}}
\newcommand{\tbw}{\widetilde{\mathbf{w}}}
\newcommand{\bw}{\mathbf{w}}
\newcommand{\by}{\mathbf{y}}
\newcommand{\brw}{\overline{\mathbf{w}}}
\newcommand{\bv}{\mathbf{v}}

\newcommand{\prox}{{\sf prox}}

\newcommand{\hprox}{\widehat{\prox}}

\newcommand{\re}{\mathbb{R}}
\newcommand{\cB}{\mathcal{B}}
\newcommand{\cZ}{\mathcal{Z}}
\newcommand{\cN}{\mathcal{N}}

\newcommand{\bzero}{\mathbf{0}}
\newcommand{\bG}{\mathbf{G}}
\newcommand{\bH}{\mathbf{H}}
\newcommand{\proj}{\mathsf{Proj}}
\newcommand{\Ansgd}{\mathcal{A}_{\sf NSGD}}
\newcommand{\Aprox}{\mathcal{A}_{\sf ProxGD}}

\newcommand{\cW}{\mathcal{W}}

\newcommand{\grad}{\bigtriangledown}
\newcommand{\mypar}[1]{\smallskip
	\noindent{\textbf{\em {#1}:}}}

\newcommand{\boldpar}[1]{\smallskip}

\newcommand{\mynote}[3]{}

 \newcommand{\knote}[1]{\mynote{}}
 \newcommand{\rnote}[1]{\mynote{}}
 \newcommand{\vnote}[1]{\mynote{}}
 \newcommand{\atnote}[1]{\mynote{}}

\newcommand{\ignore}[1]{}

\newcommand{\cO}{\mathcal{O}}

\newcommand{\param}{\mathbf{w}}

\title{Private Stochastic Convex Optimization with Optimal Rates}
\ifnum\nips=1
\author{}

\else
\author{Raef Bassily\thanks{Department of Computer Science \& Engineering, The Ohio State University. \texttt{bassily.1@osu.edu}}
	\and Vitaly Feldman\thanks{Google Research. Brain Team.} \and Kunal Talwar\thanks{Google Research. Brain Team. \texttt{kunal@google.com}.} \and
	Abhradeep Thakurta\thanks{Department of Computer Science, University of California Santa Cruz. \texttt{aguhatha@ucsc.edu}}
}
\date{}
\fi

\begin{document}

\maketitle

\begin{abstract}



We study differentially private (DP) algorithms for stochastic convex optimization (SCO). In this problem the goal is to approximately minimize the population loss given i.i.d.~samples from a distribution over convex and Lipschitz loss functions. A long line of existing work on private convex optimization focuses on the empirical loss and derives asymptotically tight bounds on the excess empirical loss.  However a significant gap exists in the known bounds for the population loss.

We show that, up to logarithmic factors, the optimal excess population loss for DP algorithms is equal to the larger of the optimal non-private excess population loss, and the optimal excess empirical loss of DP algorithms. This implies that, contrary to intuition based on private ERM, private SCO has asymptotically the same rate of $1/\sqrt{n}$ as non-private SCO in the parameter regime most common in practice. The best previous result in this setting gives rate of $1/n^{1/4}$. Our approach builds on existing differentially private algorithms and relies on the analysis of algorithmic stability to ensure generalization.


    \end{abstract}

\ifnum\nips=0
\section{Introduction}\label{sec:intro}

Many fundamental problems in machine learning reduce to the problem of minimizing the expected loss (also referred to as {\em population loss}) $\cL(\param) = \ex{z \sim \mathcal{D}}{\ell(\param,z)}$ for convex loss functions of $\param$
given access to i.i.d.~samples $z_1, \ldots, z_n$ from the data distribution $\cal D$. This problem arises in various settings, such as estimating the mean of a distribution, least squares regression, or minimizing a convex surrogate loss for a classification problem. This problem is commonly referred to as {\em stochastic convex optimization} (SCO) and has been the subject of extensive study in machine learning and optimization. In this work we study this problem with the additional constraint of differential privacy with respect to the samples \cite{DMNS06}.

A natural approach toward solving SCO is minimization of the empirical loss $\hL(\param) = \frac{1}{n} \sum_i \ell(\param,z_i)$ and is referred to as empirical risk minimization (ERM). The problem of ERM with differential privacy (DP-ERM) has been well-studied and asymptotically tight upper and lower bounds on excess loss\footnote{{\em Excess loss} refers to the difference between the achieved loss and the true minimum.} are known  \cite{CM08,CMS,jain2012differentially,kifer2012private,ST13sparse,song2013stochastic,DuchiJW13,ullman2015private,JTOpt13,bassily2014differentially,talwar2015nearly,smith2017interaction,wu2017bolt,wang2017differentially,iyengartowards}.

A standard approach for deriving bounds on the population loss is to appeal to {\em uniform convergence} of empirical loss to population loss, namely an upper bound on $\sup_{\param} (\cL(\param) - \hL(\param))$. This approach can be used to derive optimal bounds on the excess population loss in a number of special cases, such as regression for generalized linear models. However, in general, it leads to suboptimal bounds.
It is known that there exist distributions over loss functions over $\re^d$ for which the best bound on uniform convergence is $\Omega(\sqrt{d/n})$ \cite{feldman2016generalization}. In contrast, in the same setting, DP-ERM can be solved with excess loss of  $O(\frac{\sqrt{d}}{\eps n})$ and the optimal excess population loss achievable without privacy is $O(\sqrt{1/n})$.
As a result, in the high-dimensional settings often considered in modern ML (when $n = \Theta(d)$), bounds based on uniform convergence are $\Omega(1)$ and do not lead to meaningful bounds on population loss.

The first work to address the population loss for SCO with differential privacy (DP-SCO) is \cite{bassily2014differentially}. It gives bounds based on two natural approaches. 
The first approach is to use the generalization properties of differential privacy itself to bound the gap between the empirical and population losses \cite{dwork2015preserving, bassily2016algorithmic}, and thus derive bounds for SCO from bounds on ERM. This approach leads to a suboptimal bound (specifically\footnote{For clarity, in the introduction we focus on the dependence on $d$ and $n$ and $\eps$ for $(\eps, \delta)$-DP. We suppress the dependence on $\delta$ and on parameters of the loss function such as Lipschitz constant and the constraint set radius.}, $\approx \max\left(\tfrac{d^{\frac{1}{4}}}{\sqrt{n}}, \tfrac{\sqrt{d}}{\eps n}\right)$ \cite[Sec. F]{bassily2014differentially}). For the important case when $d=\Theta(n)$ and $\eps=\Theta(1)$ this results in the bound of $\Omega(n^{-\frac 1 4})$ on excess population loss. The second approach relies on generalization properties of stability to bound the gap between the empirical and population losses \cite{bousquet2002stability,SSSS}. Stability is ensured by adding a strongly convex regularizer to the empirical loss \cite{SSSS}. This technique also yields a suboptimal bound on the excess population loss $\approx (d^{\frac{1}{4}}/\sqrt{\eps\,n})$.

There are two natural lower bounds that apply to DP-SCO. 
The lower bound of $\Omega(\sqrt{1/n})$ for the excess loss of non-private SCO applies for DP-SCO.
Further it is not hard to show that lower bounds for DP-ERM translate to essentially the same lower bound for DP-SCO, leading to a lower bound of $\Omega(\frac{\sqrt{d}}{\eps n})$ (see Appendix~\ref{sec:lower} for the proof).

\subsection{Our contribution}
In this work, we address the gap between the known bounds for DP-SCO. Specifically, we show that the optimal rate of $O\left(\sqrt\frac{1}{n} + \frac{\sqrt{d}}{\eps n}\right)$ is achievable, matching the known lower bounds. In particular, we obtain the statistically optimal rate of $O(1/\sqrt{n})$ whenever $d=O(n)$. This is in contrast to the situation for DP-ERM where the cost of privacy grows with the dimension for all $n$.

In our first result we show that, under relatively mild smoothness assumptions, this rate is achieved by a variant of the standard noisy mini-batch SGD. The classical analyses for non-private SCO depend crucially on making only one pass over the dataset. However, a single pass noisy SGD is not sufficiently accurate as we need a non-trivial amount of noise in each step to carry out the privacy analysis. We rely instead on generalization properties of {\em uniform stability} \cite{bousquet2002stability}. Unlike in \cite{bassily2014differentially}, our analysis of stability is based on extension of recent stability analysis of SGD \cite{hardt2015train,feldman2019high} to noisy SGD. In this analysis, the stability parameter degrades with the number of passes over the dataset, while the empirical loss decreases as we make more passes. 
In addition, the batch size needs to be sufficiently large to ensure that the noise added for privacy is small. To satisfy all these constraints the parameters of the scheme need to be tuned carefully. Specifically we show that $\approx \min(n, n^2\eps^2/d)$ steps of SGD with a batch size of $\approx \max(\sqrt{\eps n}, 1)$
 are sufficient to get all the desired properties.

Our second contribution is to show that the smoothness assumptions can be relaxed at essentially no additional increase in the rate. We use a general smoothing technique based on the Moreau-Yosida envelope operator that allows us to derive the same asymptotic bounds as the smooth case. This operator cannot be implemented efficiently in general, but for algorithms based on gradient steps we exploit the well-known connection between the gradient step on the smoothed function and the proximal step on the original function. Thus our algorithm is equivalent to (stochastic, noisy, mini-batch) proximal descent on the unsmoothed function. We show that our analysis in the smooth case is robust to inaccuracies in the computation of the gradient. This allows us to show that sufficient approximation to the proximal steps can be implemented in polynomial time given access to the gradient of the $\ell(w,z_i)$'s.

Finally, we show that {\em Objective Perturbation} \cite{CMS,kifer2012private} also achieves optimal bounds for DP-SCO. However, objective perturbation is only known to satisfy privacy under some additional assumptions, most notably, Hessian being rank $1$ on all points in the domain. The generalization analysis in this case is based on the uniform stability of the solution to strongly convex ERM. Aside from extending the analysis of this approach to population loss, we show that it can lead to algorithms for DP-SCO that use only near-linear number of gradient evaluations (whenever these assumptions hold). In particular, we give a variant of objective perturbation in conjunction with the stochastic variance reduced gradient descent (SVRG) with only $O(n\log n)$ gradient evaluations. We remark that the known lower bounds for uniform convergence \cite{feldman2016generalization} hold even under those additional assumptions invoked in objective perturbation. Finding algorithms with near-linear running time in the general setting of SCO is a natural avenue for future research. 

Our work highlights the importance of uniform stability as a tool for analysis of this important class of problems. We believe it should have applications to other differentially private statistical analyses.











\else

\fi
\paragraph{Related work:}
Differentially private empirical risk minimization (ERM) is a well-studied area spanning over a decade
\cite{CM08,CMS,jain2012differentially,kifer2012private,ST13sparse,song2013stochastic,DuchiJW13,ullman2015private,JTOpt13,bassily2014differentially,talwar2015nearly,smith2017interaction,wu2017bolt,wang2017differentially,iyengartowards}. Aside from ~\cite{bassily2014differentially} and work in the local model of DP \cite{DuchiJW13} these works focus on achieving optimal \emph{empirical} risk bounds under privacy. Our work builds heavily on algorithms and analyses developed in this line of work while contributing additional insights.
\ifnum\nips=1
Optimal bounds for private SCO are known for some simple subclasses of convex functions such as Generalized Linear Models~\cite{JTOpt13,bassily2014differentially} where uniform convergence bounds on the order of $1/\sqrt{n}$ are known \cite{KakadeST:08}.
\fi

\ifnum\nips=0

\section{Preliminaries}\label{sec:prelim}



\paragraph{Notation:} We use $\cW\subset \re^d$ to denote the parameter space, which is assumed to be a convex, compact set. We denote by $M = \max\limits_{\bw\in\cW}\norm{\bw}$ the $L_2$ radius of $\cW$. We use $\Z$ to denote an arbitrary data domain and $\D$ to denote an arbitrary distribution over $\Z$. We let $\ell:\re^d \times \Z\rightarrow \re$ be a loss function that takes a parameter vector $\bw\in\W$ and a data point $z\in\Z$ as inputs and outputs a real value.




The \emph{empirical loss} of $\bw\in\W$ w.r.t. loss $\ell$ and dataset $S=(z_1, \ldots, z_n)$ is defined as $\hL(\bw;~ S)\triangleq \frac{1}{n}\sum_{i=1}^n \ell(\bw,z_i).$
The \emph{excess empirical loss} of $\bw$ is defined as $\hL(\bw;~ S)-\min\limits_{\tbw\in\W}\hL\left(\tbw; ~S\right).$
The \emph{population loss} of $\bw\in\W$ with respect to a loss $\ell$ and a distribution $\D$ over $\Z$, 
is defined as $\cL(\bw; \D)\triangleq \ex{z\sim \D}{\ell(\bw,z)}.$
The \emph{excess population loss} of $\bw$ is defined as $\cL(\bw;~\D)-\min\limits_{\tbw\in\W}\cL(\tbw;~\D).$

\begin{defn}[Uniform stability]\label{def:unif-stable}
Let $\alpha>0$. A (randomized) algorithm $\A: \cZ^n\rightarrow \cW$ is $\alpha$-uniformly stable (w.r.t. loss $\ell:\cW\times\cZ\rightarrow \re$) if for any pair $S, ~S' \in\cZ^n$ differing in at most one data point, we have
$$\sup\limits_{z\in\cZ}\,\ex{\A}{\ell\left(\A(S), z\right)-\ell\left(\A(S'), z\right)}\leq \alpha$$
where the expectation is taken only over the internal randomness of $\A$.
\end{defn}

We will use the following simple generalization property of stability that upper bounds the expectation of population loss. Our bounds on excess population loss can also be shown to hold (up to log factors) with high probability using the results from \cite{feldman2019high}.
\begin{lem}[\cite{bousquet2002stability}]\label{lem:gen_err_stability}
Let $\A:\Z^n\rightarrow\W$ be an $\alpha$-uniformly stable algorithm w.r.t. loss $\ell:\W\times\Z\rightarrow\re$. Let $\D$ be any distribution over $\Z$, and let $S\sim\D^n$. Then,
\begin{align*}
    \ex{S\sim\D^n, \A}{\cL\left(\A(S);~\D\right)-\hL\left(\A(S); ~ S\right)}&\leq \alpha.
\end{align*}
\end{lem}




\begin{defn}[Smooth function]\label{defn:smooth}
Let $\beta>0$. A differentiable function $f:\re^d\rightarrow \re$ is $\beta$-smooth over $\W \subseteq \re^d$ if for every $\bw, \bv \in\W,$ we have
$$f(\bv)\leq f(\bw)+\langle \nabla f(\bw), \bv-\bw\rangle + \frac{\beta}{2} \,\norm{\bw-\bv}^2.$$

\end{defn}


In the sequel, whenever we attribute a property (e.g., convexity, Lipschitz property, smoothness, etc.) to a loss function $\ell$, we mean that for every data point $z\in\Z,$ the loss $\ell(\cdot, z)$ possesses that property over $\W$.


\paragraph{Stochastic Convex Optimization (SCO):} Let $\D$ be an arbitrary (unknown) distribution over $\Z$, and $S=(z_1, \ldots, z_n)$ be a sequence of i.i.d.~samples from $\D$. Let $\ell:\W\times\Z\rightarrow \re$ be a convex loss function. A (possibly randomized) algorithm for SCO uses the sample $S$ to generate an (approximate) minimizer $\hbw_S$ for $\cL(\cdot; ~\D)$. We measure the accuracy of $\A$ by the \emph{expected} excess population loss of its output parameter $\hbw_S$, defined as:
$$\er\left(\A;~\D\right)\triangleq \ex{}{\cL(\hbw_S;~\D)-\min\limits_{\bw\in\W}\cL(\bw; ~\D)},$$
where the expectation is taken over the choice of $S\sim\D^n$, and any internal randomness in $\A$.

\paragraph{Differential privacy \cite{DMNS06, DKMMN06}:} A randomized  algorithm $\A$ is $(\eps,\delta)$-differentially private if, for any pair of datasets $S$ and $S'$ differ in exactly one data point, and for all events $\cO$ in the output range of $\A$, we have
$$\pr{}{\A(S)\in \cO} \leq e^{\eps} \cdot \pr{}{\A(S')\in \cO} +\delta ,$$
where the probability is taken over the random coins of $\A$. For meaningful privacy guarantees, the typical settings of the privacy parameters are $\eps<1$ and $\delta \ll 1/n$.

\paragraph{Differentially Private Stochastic Convex Optimization (DP-SCO):} An $(\eps, \delta)$-DP-SCO algorithm is a SCO algorithm that satisfies $(\eps, \delta)$-differential privacy.


\section{Private SCO via Mini-batch Noisy SGD}\label{sec:smooth}

In this section, we consider the setting where the loss $\ell$ is convex, Lipschitz, and smooth. We give a technique that is based on a mini-batch variant of Noisy Stochastic Gradient Descent (NSGD) algorithm \cite{bassily2014differentially, abadi2016deep} described in Figure \ref{Alg:NSGD-smooth}.

\begin{algorithm}
	\caption{$\A_{\sf NSGD}$: Mini-batch noisy SGD for convex, smooth losses}
	\begin{algorithmic}[1]
		\REQUIRE Private dataset: $S=(z_1, \ldots, z_n)\in \cZ^n$, $L$-Lipschitz, $\beta$-smooth, convex loss function $\ell$, convex set $\cW\subseteq \re^d$, step size $\eta$, mini-batch size $m$, ~\# iterations $T$, privacy parameters $\epsilon \leq 1,\,  \delta \leq 1/n^2$.
				\STATE Set noise variance $\sigma^2 := \frac{8T\,L^2\,\log(1/\delta)}{n^2\epsilon^2}.$
				\STATE Set batch size $m:=\max\left(n\,\sqrt{\frac{\epsilon}{4\,T}},~ 1\right).$\label{step:batch-size}
				\STATE Choose arbitrary initial point $\bw_0 \in \cW.$
        \FOR{$t=0$ to $T-1$\,}
        	\STATE Sample a batch $B_t=\{z_{i_{(t, 1)}}, \ldots, z_{i_{(t, m)}}\}\leftarrow S$ uniformly with replacement.\label{step:sampling}
        	\STATE $\bw_{t+1} := \proj_{\cW}\left(\bw_{t}-\eta\cdot\left(\frac{1}{m}\sum_{j=1}^m\nabla\ell(\bw_{t}, z_{i_{(t,j)}})+\bG_t\right)\right),$ where $\proj_{\cW}$ denotes the Euclidean projection onto $\cW$, and $\bG_t \sim \cN\left(\bzero, \sigma^2 \mathbb{I}_d\right)$ drawn independently each iteration.\label{step:grad-step}
            \ENDFOR
            \RETURN $\brw_T=\frac{1}{T}\sum_{t=1}^{T}\bw_t$
	\end{algorithmic}
	\label{Alg:NSGD-smooth}
\end{algorithm}

\begin{thm}[Privacy guarantee of $\A_{\sf NSGD}$]
Algorithm~\ref{Alg:NSGD-smooth} is $(\epsilon, \delta)$-differentially private.
\end{thm}
\begin{proof}
The proof follows from \cite[Theorem~1]{abadi2016deep}, which gives a tight privacy analysis for mini-batch NSGD via the Moments Accountant technique and privacy amplification via sampling. We note that the setting of the mini-batch size in Step~\ref{step:batch-size} of Algorithm~\ref{Alg:NSGD-smooth} satisfies the condition in \cite[Theorem~1]{abadi2016deep} (we obtain here an explicit value for the universal constants in the aforementioned theorem in that reference). We also note that the setting of the Gaussian noise in \cite{abadi2016deep} is not normalized by the mini-batch size, and hence the noise variance reported in \cite[Theorem~1]{abadi2016deep} is larger than our setting of $\sigma^2$ by a factor of $m^2$.
\end{proof}

The population loss attained by $\Ansgd$ is given by the next theorem.

\begin{thm}[Excess population loss of $\Ansgd$]\label{thm:pop_risk_Ansgd}

Let $\D$ be any distribution over $\Z,$ and let $S\sim\D^n$. Suppose $\beta \leq \frac{L}{M}\cdot \min\left(\sqrt{\frac{n}{2}}, \frac{\epsilon\,n}{2\sqrt{2 d\log(1/\delta)}}\right)$. Let $T =\min\left(\frac{n}{8},~ \frac{\epsilon^2\,n^2}{32\,d\,\log(1/\delta)}\right)$ and $\eta= \frac{M}{L\,\sqrt{T}}$. Then,
\begin{align*}
    \er\left(\A_{\sf NSGD};~\D\right)&\leq 10\,ML\cdot\max\left(\frac{\sqrt{d\,\log(1/\delta)}}{\epsilon\, n},~\frac{1}{\sqrt{n}}\right)
\end{align*}
\end{thm}

Before proving the above theorem, we first state and prove the following useful lemmas.

\begin{lem}\label{lem:emp-risk}
Let $S\in \cZ^n$. Suppose the parameter set $\cW$ is convex and $M$-bounded. For any $\eta > 0,$ the excess empirical loss of $\A_{\sf NSGD}$ satisfies

\begin{align*}
    \ex{}{\hL(\brw_T; S)}-\min\limits_{\bw\in\cW}\hL(\bw; S)&\leq \frac{M^2}{2\,\eta\, T} + \frac{\eta\,L^2}{2} \left(16\frac{T\,d\,\log(1/\delta)}{n^2\,\epsilon^2}+1\right)
\end{align*}
where the expectation is taken with respect to the choice of the mini-batch (step~\ref{step:sampling}) and the independent Gaussian noise vectors $\bG_1, \ldots, \bG_T$.
\end{lem}

\begin{proof}
The proof follows from the classical analysis of the stochastic oracle model (see, e.g., \cite{shalev2014understanding}).
In particular, we can show that
$$\ex{}{\hL(\brw_T; S)}-\min\limits_{\bw\in\cW}\hL(\bw; S)\leq \frac{M^2}{2\,\eta\, T} + \frac{\eta\,L^2}{2} +\eta\,\sigma^2\, d,$$
where the last term captures the additional empirical error due to privacy. The statement now follows from the setting of $\sigma^2$ in Algorithm~\ref{Alg:NSGD-smooth}.
\end{proof}

The following lemma is a simple extension of the results on uniform stability of GD methods that appeared in \cite{hardt2015train} and \cite[Lemma~4.3]{feldman2019high} to the case of \emph{mini-batch noisy} SGD. For completeness, we provide a proof in Appendix~\ref{app:A}.


\begin{lem}\label{lem:stability}
In $\A_{\sf NSGD}$, suppose $\eta \leq \frac{2}{\beta},$ where $\beta$ is the smoothness parameter of $\ell$. Then, $\A_{\sf NSGD}$ is $\alpha$-uniformly stable with $\alpha= L^2\frac{T\,\eta}{n}$.
\end{lem}

\subsubsection*{Proof of Theorem~\ref{thm:pop_risk_Ansgd}}
By Lemma~\ref{lem:gen_err_stability}, $\alpha$-uniform stability implies that the expected population loss is upper bounded by $\alpha$ plus the expected empirical loss. Hence, by combining Lemma~\ref{lem:emp-risk} with Lemma~\ref{lem:stability}, we have
\begin{align}
\ex{S\sim\D^n,~\Ansgd}{\cL(\brw_T;~\D)}-\min\limits_{\bw\in\W}\cL(\bw; ~\D)&\leq \ex{S\sim\D^n,~\Ansgd}{\hL(\brw_T; S)}-\min\limits_{\bw\in\W}\cL(\bw; ~\D)+L^2\,\frac{\eta\,T}{n}\nonumber\\
&\leq \ex{S\sim\D^n,~\Ansgd}{\hL(\brw_T; S)-\min\limits_{\bw\in\cW}\hL(\bw; S)}+L^2\,\frac{\eta\,T}{n}\label{ineq:emp-less-pop}\\
&\leq \frac{M^2}{2\,\eta\, T} + \frac{\eta\,L^2}{2} \left(16\frac{T\,d}{n^2\,\epsilon^2}+1\right)+ L^2\,\frac{\eta\,T}{n}\nonumber
\end{align}
where (\ref{ineq:emp-less-pop}) follows from the fact that $\ex{S\sim\D^n}{\min\limits_{\bw\in\cW}\hL(\bw; S)}\leq \min\limits_{\bw\in\W}\ex{S\sim\D^n}{\hL(\bw; S)}=\min\limits_{\bw\in\W}\cL(\bw; ~\D)$.
Optimizing the above bound in $\eta$ and $T$ yields the values in the theorem statement for these parameters, as well as the stated bound on the excess population loss.

\section{Private SCO for Non-smooth Losses}\label{sec:non-smooth}

In this section, we consider the setting where the convex loss is non-smooth. First, we show a generic reduction to the smooth case by employing the smoothing technique known as \emph{Moreau-Yosida regularization} (a.k.a. Moreau envelope smoothing) \cite{nesterov2005smooth}. Given an appropriately smoothed version of the loss, we obtain the optimal population loss w.r.t. the original non-smooth loss function. Computing the smoothed loss via this technique is generally computationally inefficient. Hence, we move on to describe a computationally efficient algorithm for the non-smooth case with essentially optimal population loss. Our construction is based on an adaptation of our noisy SGD algorithm $\Ansgd$ (Algorithm~\ref{Alg:NSGD-smooth}) that exploits some useful properties of Moreau-Yosida smoothing technique that stem from its connection to proximal operations.

\begin{defn}[Moreau envelope]\label{defn:moreau}
Let $f:\W\rightarrow \re^d$ be a convex function, and $\beta>0$. The $\beta$-Moreau envelope of $f$ is a function $f_{\beta}:\W\rightarrow\re^d$ defined as
$$f_{\beta}(\bw)=\min\limits_{\bv\in\W}\left(f(\bv)+\frac{\beta}{2}\norm{\bw-\bv}^2\right), \quad \bw\in\W.$$
\end{defn}

Moreau envelope has direct connection with the proximal operator of a function defined below.
\begin{defn}[Proximal operator]\label{defn:prox}
The prox operator of $f:\W\rightarrow\re^d$ is defined as
$$\prox_f(\bw)=\arg\min\limits_{\bv\in\W}\left(f(\bv)+\frac{1}{2}\norm{\bw-\bv}^2\right), \quad \bw\in\W.$$
\end{defn}

It follows that the Moreau envelope $f_{\beta}$ can be written as $$f_{\beta}(\bw)=f\left(\prox_{f/\beta}\left(\bw\right)\right)+\frac{\beta}{2}\norm{\bw-\prox_{f/\beta}\left(\bw\right)}^2.$$

The following lemma states some useful, known  properties of Moreau envelope.

\begin{lem}[See \cite{nesterov2005smooth, candes_opt_notes}]\label{lem:prop_moreau}
Let $f:\W\rightarrow \re^d$ be a convex, $L$-Lipschitz function, and let $\beta>0$. The $\beta$-Moreau envelope $f_{\beta}$ satisfies the following:
\begin{enumerate}
    \item $f_{\beta}$ is convex, $2L$-Lipschitz, and $\beta$-smooth.\label{prop:moreau_1}
    \item $\forall \bw\in \W\quad f_{\beta}(\bw)\leq f(\bw)\leq f_{\beta}(\bw)+\frac{L^2}{2\,\beta}.$\label{prop:moreau_2}
    \item $\forall \bw\in\W \quad \nabla f_{\beta}(\bw)=\beta\, \left(\bw - \prox_{f/\beta}(\bw)\right).$\label{prop:moreau_3}
\end{enumerate}
\end{lem}

The convexity and $\beta$-smoothness together with properties~\ref{prop:moreau_2} and \ref{prop:moreau_3} are fairly standard and the proof can be found in the aforementioned references. The fact that $f_{\beta}$ is $2L$-Lipschitz follows easily from property~\ref{prop:moreau_3}. We include the proof of this fact in Appendix~\ref{app:B} for completeness.

Let $\ell:\W\times\Z\rightarrow\re$ be a convex, $L$-Lipschitz loss. For any $z\in\Z,$ let $\lbeta(\cdot, z)$ denote the $\beta$-Moreau envelope of $\ell(\cdot,~z).$ For a dataset $S=(z_1, \ldots, z_n)\in\Z^n,$ let $\hL_{\beta}(\cdot;~ S)\triangleq \frac{1}{n}\sum_{i=1}^n\lbeta(\cdot,~z_i)$ be the empirical risk w.r.t. the $\beta$-smoothed loss. For any distribution $\D$, let $\cL_{\beta}(\cdot; \D)\triangleq \ex{z\sim \D}{\ell_\beta(\cdot,~z)}$ denote the corresponding population loss. The following theorem asserts that, with an appropriate setting for $\beta,$ running $\Ansgd$ over the $\beta$-smoothed losses $\lbeta(\cdot, z_i), ~i\in [n]$ yields the optimal population loss w.r.t.~the original non-smooth loss $\ell$.

\begin{thm}[Excess population loss for non-smooth losses via smoothing]\label{thm:pop_risk_non-smooth_generic}
Let $\D$ be any distribution over $\Z$. Let $S=(z_1, \ldots, z_n)\sim\D^n$. Let $\beta = \frac{L}{M}\cdot \min\left(\frac{\sqrt{n}}{4}, \frac{\epsilon\,n}{8\sqrt{d\,\log(1/\delta)}}\right).$ Suppose we run $\Ansgd$ (Algorithm~\ref{Alg:NSGD-smooth}) over the $\beta$-smoothed version of $\ell$ associated with the points in $S$: $\left\{\lbeta(\cdot, z_i),~i\in[n]\right\}$. Let $\eta$ and $T$ be set as in Theorem~\ref{thm:pop_risk_Ansgd}. Then, the excess population loss of the output of $\Ansgd$ w.r.t. $\ell$ satisfies
\begin{align*}
    \er\left(\Ansgd; \D\right)&\leq 24\,M\,L\cdot\max\left(\frac{\sqrt{d\,\log(1/\delta)}}{\epsilon\, n},~\frac{1}{\sqrt{n}}\right)
\end{align*}
\end{thm}
\begin{proof}
Let $\brw_T$ be the output of $\Ansgd$. Using property~\ref{prop:moreau_1} of Lemma~\ref{lem:prop_moreau} together with Theorem~\ref{thm:pop_risk_Ansgd}, we have
$$\ex{S\sim\D^n, \Ansgd}{\cL_{\beta}(\brw_T; \D)}-\min\limits_{\bw\in\W}\cL_{\beta}(\bw;~\D)\leq 20\,M\,L\cdot\max\left(\frac{\sqrt{d\,\log(1/\delta)}}{\epsilon\, n},~\frac{1}{\sqrt{n}}\right).$$
Now, by property~\ref{prop:moreau_2} of Lemma~\ref{prop:moreau_2} and the setting of $\beta$ in the theorem statement, for every $\bw\in \W$, we have
$$\cL_{\beta}(\bw;~\D)\leq \cL(\bw;~\D)\leq \cL_{\beta}(\bw;~\D)+2\,M\,L\cdot\max\left(\frac{1}{\sqrt{n}}, ~\frac{2 \sqrt{d\,\log(1/\delta)}}{\eps\, n}\right).$$
Putting these together gives the stated result.
\end{proof}

\subsection*{Computationally efficient algorithm $\Aprox$ ({\sf NSGD} + {\sf Prox})}
Computing the Moreau envelope of a function is  computationally inefficient in general. However, by property~\ref{prop:moreau_3} of Lemma~\ref{lem:prop_moreau}, we note that evaluating the gradient of Moreau envelope at any point can be attained by evaluating the proximal operator of the function at that point. Evaluating the proximal operator is equivalent to minimizing a strongly convex function (see Definition~\ref{defn:prox}). This can be approximated efficiently, e.g., via gradient descent. Since our $\Ansgd$ algorithm (Algorithm~\ref{Alg:NSGD-smooth}) requires only sufficiently accurate gradient evaluations, we can hence use an efficient, approximate proximal operator to approximate the gradient of the smoothed losses. The gradient evaluations in $\Ansgd$ will thus be replaced with such approximate gradients evaluated via the approximate proximal operator. The resulting algorithm, referred to as $\Aprox$, will approximately minimize the smoothed empirical loss without actually computing the smoothed losses. 

\begin{defn}[Approximate $\prox$ operator]\label{defn:approx-prox}
We say that $\hprox_f$ is an $\xi$-approximate proximal operator of $\prox_f$ for a function $f:\W\rightarrow\re$ if $~\forall \bw\in\W,~ \norm{\hprox_f(\bw)-\prox_f(\bw)}\leq \xi.$
\end{defn}

\begin{fact}\label{fact:prox-approx-err-conv}
Let $\W\subset\re^d$ be $M$-bounded. Let $f:\W\rightarrow\re$ be convex, $L$-Lipschitz function. Suppose $\beta\geq \frac{L}{M}$. For all $\xi>0$, there is $\xi$-approximate $\hprox_{f/\beta}$ such that for each $\bw\in\W$, computing $\hprox_{f/\beta}(\bw)$ requires time that is equivalent to at most $\lceil\frac{8\,M^2}{\xi^2}\rceil$ gradient evaluations.
\end{fact}

This fact follows from the fact that $\prox_{f/\beta}(\bw)=\arg\min\limits_{\bv\in\W}g_{\bw}(\bv),$ where $g_{\bw}(\bv)\triangleq\frac{1}{\beta}\,f(\bv) + \frac{1}{2}\norm{\bv -\bw}^2$. This is minimization of $1$-strongly convex and $2\,M$-Lipschitz function over $\W$, The Lipschitz constant follows from the fact that $\beta\geq L/M$. Hence, one can run ordinary Gradient Descent to obtain an approximate minimizer. From a standard result on convergence of GD for strongly convex and Lipschitz functions \cite{bubeck2015convex}, in $\tau$ gradient steps we obtain an approximate  $\bv_{\tau}$ satisfying $g_{\bw}(\bv_{\tau})-g_{\bw}(\bv^*)\leq \frac{8\,M^2}{\tau}$, where $\bv^*=\arg\min\limits_{\bv\in\W}g_{\bw}(\bv)$. Since $g_{\bw}$ is $1$-strongly convex, we get $\norm{\bv_{\tau}-\bv^*}\leq \sqrt{\frac{8\,M^2}{\tau}}$.

\paragraph{Description of $\Aprox$:} The algorithm description follows exactly the same lines as $\Ansgd$ except that: (i) the input loss $\ell$ is now non-smooth, and (ii) for each iteration $t$, the gradient evaluation $\nabla \ell(\bw_t, z)$ for each data point $z$ in the mini-batch is replaced with the evaluation of an approximate gradient of the smoothed loss $\ell_{\beta}(\cdot, z)$. The approximate gradient, denoted as $\hnabla \ell_{\beta}(\bw_t, z)$, is computed using an approximate proximal operator. Namely,
$$\hnabla \ell_{\beta}(\bw_t, z):= \beta\cdot\left(\bw_t-\hprox_{\ell_z/\beta}(\bw_t)\right),$$
where $\ell_z\triangleq \ell(\cdot,~z)$. Here, we use a computationally efficient $\xi$-approximate $\hprox_{\ell_z/\beta}$ like the one in Fact~\ref{fact:prox-approx-err-conv} with $\xi$ set as
$$\xi:= 4\,\frac{M}{n}\cdot\max\left(\frac{2\,\sqrt{d\,\log(1/\delta)}}{\eps\,n},~\frac{1}{\sqrt{n}}\right).$$
Note that the approximation error in the gradient $\norm{\hnabla\ell_{\beta}(\bw_t, z)-\nabla\ell_{\beta}(\bw_t, z)}\leq \beta\cdot\xi$, and that $\beta\cdot\xi=\frac{L}{n},$ where $L$ is the Lipschitz constant of $\ell$.

\paragraph{Running time of $\Aprox$:} if we use the approximate proximal operator in Fact~\ref{fact:prox-approx-err-conv}, then it is easy to see that $\Aprox$ requires a number of gradient evaluations that is a factor of $n^2\,T$ more than $\Ansgd$, where $T=O\left(\max\left(n,~\frac{\eps^2\,n^2}{d\,\log(1/\delta)}\right)\right).$ That is, the total number of gradient evaluations is $n^2\cdot T^2\cdot m,$ where $m=O\left(\max\left(\sqrt{\eps\,n},~\sqrt{\frac{d\,\log(1/\delta)}{\eps}}\right)\right)$ is the mini-batch size.

We now argue that privacy, stability, and accuracy of the algorithm are preserved under the approximate proximal operator.

\paragraph{Privacy:} Note that to bound the sensitivity of the approximate gradient of the mini-batch, it suffices to bound the norm of the approximate gradient. From the discussion above, note that $\forall ~z, \forall~ \bw\in\W,$ we have $\norm{\hnabla\ell_{\beta}(\bw, z)}\leq \norm{\hnabla\ell_{\beta}(\bw, z)-\nabla\ell_{\beta}(\bw, z)}+\norm{\nabla\ell_{\beta}(\bw, z)}\leq L\,(1+\frac{1}{n}).$ Thus, the sensitivity remains basically the same as in the case where the algorithm is run with the exact gradients. Hence, the same privacy guarantee holds as in $\Ansgd$.

\paragraph{Empirical error:} Note that the approximation error in the gradient of the mini-batch (due to the approximate proximal operation) can be viewed as a \emph{fixed} error term of magnitude at most $\frac{L}{n}$ that is added to the exact gradient of the smoothed loss. It is well-known and easy to see that the effect of this additional approximation error on the standard convergence bounds is that excess empirical loss may grow by at most the error times the diameter of the domain (e.g.~\citep{nedic2010effect,FeldmanGV:15}). Hence, compared to the error bound error in Lemma~\ref{lem:emp-risk}, the bound we get incurs an additional term of $2LM/n$. Clearly, this additional error is dominated by the other terms in the empirical loss bound in Lemma~\ref{lem:emp-risk}, and thus will have no significant impact on the final bound.


\paragraph{Uniform stability:} This easily follows from the following facts. First, note that the additional approximation error due to gradient approximation is $\frac{L}{n}$. Second, the gradient update w.r.t.~the exact gradient of the smoothed loss is non-expansive operation (which is the key fact in proving uniform stability of (stochastic) gradient methods \cite{hardt2015train, feldman2019high}), and hence the approximation error in the gradient is not going to be amplified by the gradient update step. Hence, for any trajectory of $T$ approximate gradient updates, the accumulated approximation error in the final output $\brw_T$ cannot exceed $\frac{T\,\eta\,L}{n}$. This cannot increase the final uniform stability bound by more than an additive term of $\frac{T\,\eta\,L^2}{n}$. Thus, we obtain basically the same bound in Lemma~\ref{lem:stability}.

Putting these together, we have argued that $\Aprox$ is computationally efficient algorithm that achieves the optimal population loss bound in Theorem~\ref{thm:pop_risk_non-smooth_generic}. 

\section{Private SCO via Objective Perturbation}

In this section, we show that the technique known as objective perturbation \cite{CMS, kifer2012private} can be used to attain optimal \emph{population} loss for a large subclass of convex, smooth losses. In objective perturbation, the empirical loss is first perturbed by adding two terms: a \emph{noisy} linear term and a regularization term. As shown in \cite{CMS, kifer2012private}, under some additional assumptions on the Hessian of the loss, an appropriate random perturbation ensures differential privacy. The excess \emph{empirical} loss of this technique for smooth convex losses was originally analyzed in the aforementioned works, and was shown to be optimal by the lower bound in \cite{bassily2014differentially}. We revisit this technique and show that the regularization term added for privacy can be used to attain the optimal excess population loss by exploiting the stability-inducing property of regularization.

In addition to smoothness and convexity of $\ell$, as in \cite{CMS, kifer2012private}, we also make the following assumption on the loss function.

\begin{assumption}\label{assump:twice-diff}
For all $z\in\Z,~ \ell(\cdot,~z)$ is twice-differentiable, and the rank of its Hessian $\nabla^2 \ell(\bw, z)$ at any $\bw\in\W$ is at most $1$.

\end{assumption}
The description of the objective perturbation algorithm $\Aobj$ is given in Algorithm~\ref{Alg:ObjP-smooth}. The outline of the algorithm is the same as the one in \cite{kifer2012private} for the case of $(\eps, \delta)$-differential privacy.

\begin{algorithm}
	\caption{$\A_{\sf ObjP}$: Objective Perturbation for convex, smooth losses}
	\begin{algorithmic}[1]
		\REQUIRE Private dataset: $S=(z_1, \ldots, z_n)\in \cZ^n$, $L$-Lipschitz, $\beta$-smooth, convex loss function $\ell$, convex set $\cW\subseteq \re^d$, privacy parameters $\epsilon \leq 1,\,  \delta \leq 1/n^2$, regularization parameter $\lambda$.
			\STATE Sample $\bG\sim \cN\left(\bzero, \sigma^2\,\iden_d\right),$ where $\sigma^2= \frac{10\,L^2\,\log(1/\delta)}{\eps^2}$	
            \RETURN $\hbw=\arg\min\limits_{\bw\in\W} \hL\left(\bw;~S\right)+\frac{\langle \bG, ~\bw\rangle}{n}+\lambda\norm{\bw}^2,$ where $\hL(\bw; ~S)\triangleq \frac{1}{n}\sum_{i=1}^n\ell(\bw, ~z_i).$
	\end{algorithmic}
	\label{Alg:ObjP-smooth}
\end{algorithm}
\paragraph{Note:} The regularization term as appears in $\Aobj$ is of different scaling than the one that appears in \cite{kifer2012private}. In particular, the regularization term in \cite{kifer2012private} is normalized by $n$, whereas here it is not.  Hence, whenever the results from \cite{kifer2012private} are used here, the regularization parameter in their statements should be replaced with $n\lambda$. This presentation choice is more consistent with literature on regularization.

The privacy guarantee of $\Aobj$ is given in the following theorem, which follows directly from \cite{kifer2012private}.

\begin{thm}[Privacy guarantee of $\Aobj$, restatement of Theorem~2 in \cite{kifer2012private}]\label{thm:priv_Aobj}
Suppose that Assumption~\ref{assump:twice-diff} holds and that the smoothness parameter satisfies $\beta \leq \eps\,n\,\lambda$. Then, $\Aobj$ is $(\eps, \delta)$-differentially private.
\label{thm:privObjPert}
\end{thm}

We now state our main result for this section showing that, with appropriate setting for $\lambda$, $\Aobj$ yields asymptotically optimal excess population loss.

\begin{thm}[Excess population loss of $\Aobj$]\label{thm:pop_risk_Aobj}
Let $\D$ be any distribution over $\Z,$ and let $S\sim\D^n$. Suppose that Assumption~\ref{assump:twice-diff} holds. Suppose that $\W$ is $M$-bounded. In $\Aobj,$ set $\lambda= \frac{2\,L}{M}\sqrt{\frac{2}{n}+\frac{4\,d\,\log(1/\delta)}{\eps^2\, n^2}}.$ Then, we have
\begin{align*}
    \er\left(\Aobj;~\D\right)&\leq 2\,M\,L\,\sqrt{\frac{2}{n}+\frac{4\,d\,\log(1/\delta)}{\eps^2\, n^2}}=O\left(M\,L\cdot \max\left(\frac{1}{\sqrt{n}},~ \frac{\sqrt{d\,\log(1/\delta)}}{\eps\, n}\right)\right).
\end{align*}
\end{thm}

\paragraph{Note:} According to Theorem~\ref{thm:priv_Aobj}, $(\eps, \delta)$-differential privacy of $\Aobj$ entails the assumption that $\beta \leq \eps\, n\, \lambda.$ With the setting of $\lambda$ in Theorem~\ref{thm:pop_risk_Aobj}, it would suffice to assume that $\beta \leq \frac{2\,\eps\, L}{M}\sqrt{2\,n+4\,d\,\log(1/\delta)}.$

To prove the above theorem, we use the following lemmas.

\begin{lem}[Excess empirical loss of $\Aobj$, restatement of Theorem~26 in \cite{kifer2012private}]\label{lem:emp_risk_Aobj}
Let $S\sim\Z^n$. Under Assumption~\ref{assump:twice-diff}, the excess empirical loss of $\Aobj$ satisfies
\begin{align*}
    \ex{}{\hL(\hbw; S)}-\min\limits_{\bw\in\cW}\hL(\bw; S)&\leq \frac{16\, L^2\,d\,\log(1/\delta)}{n^2\,\epsilon^2\,\lambda}+\lambda\,M^2.
\end{align*}
where the expectation is taken over the Gaussian noise in $\Aobj$.
\end{lem}


The next lemma states the well-known stability property of regularized empirical risk minimization.

\begin{lem}[\cite{shalev2014understanding}]\label{lem:gen-error-regularize}
Let $f:\W\times\Z\rightarrow \re$ be a convex, $\rho$-Lipschitz loss, and let $\lambda>0$. Let $S=(z_1, \ldots, z_n)\sim\Z^n$. Let $\A$ be an algorithm that outputs $\tbw=\arg\min\limits_{\bw\in\W} \left(\hF(\bw;~ S)+\lambda\,\norm{\bw}^2\right),$ where $\hF(\bw;~S)=\frac{1}{n}\sum_{i=1}^n f(\bw,~z_i).$ Then, $\A$ is $\frac{2\,\rho^2}{\lambda\,n}$-uniformly stable.

\end{lem}



\subsubsection*{Proof of Theorem~\ref{thm:pop_risk_Aobj}}
Fix any realization of the noise vector $\bG$. For every $\bw\in\W, z\in\Z,$ define $f_{\bG}(\bw, z)\triangleq \ell(\bw,~z)+\frac{\langle \bG, \bw\rangle}{n}.$ Note that $f_{\bG}$ is $\left(L+\frac{\norm{\bG}}{n}\right)$-Lipschitz. For any dataset $S=(z_1, \ldots, z_n)\in\Z^n,$ define $\hF_{\bG}(\bw; S)\triangleq \frac{1}{n}\sum_{i=1}^n f_{\bG}(\bw, z_i).$ Hence, the output $\hbw$ of $\Aobj$ on input dataset $S$ can be written as $\hbw=\arg\min\limits_{\bw\in\W}\hF_{\bG}(\bw;~S)+\lambda\,\norm{\bw}^2$. Define $\cF_{\bG}(\bw;~\D)\triangleq \ex{z\sim \D}{f_{\bG}(\bw,~z)}.$ Thus, for any fixed $\bG,$ by combining Lemma~\ref{lem:gen-error-regularize} with Lemma~\ref{lem:gen_err_stability}, we have $\ex{S\sim\D^n}{\cF_{\bG}(\hbw;~\D)-\hF_{\bG}(\hbw;~S)}\leq \frac{2\,\left(L+\frac{\norm{\bG}}{n}\right)^2}{\lambda\,n}.$ On the other hand, note that for any dataset $S,$ we always have $\cF_{\bG}(\hbw;~\D)-\hF_{\bG}(\hbw;~S)=\cL(\hbw;~\D)-\hL(\hbw;~S)$ since the linear term cancels out. Hence, the expected generalization error (w.r.t. $S$) satisfies
\begin{align}
    \ex{S\sim\D^n}{\cL(\hbw;~\D)-\hL(\hbw;~S)}&\leq 2\,\frac{\left(L+\frac{\norm{\bG}}{n}\right)^2}{\lambda\,n}\nonumber
\end{align}
Now, by taking expectation over $\bG\sim\cN\left(\bzero, \sigma^2\iden_d\right)$ as well, we arrive at
\begin{align}
    \ex{}{\cL(\hbw;~\D)-\hL(\hbw;~S)}&\leq 2\,L^2\,\frac{\left(1+\frac{\sqrt{10\,d\,\log(1/\delta)}}{\eps\,n}\right)^2}{\lambda\,n}\leq 8\,\frac{L^2}{\lambda\,n}\label{ineq:gen-error-bound}
\end{align}
where we assume $\frac{\sqrt{10\,d\,\log(1/\delta)}}{\eps\, n}\leq 1$ (since otherwise we would have the trivial error).

Now, observe that:
\begin{align*}
    \er\left(\Aobj; \D\right)&=\ex{}{\cL(\hbw; \D)}-\min\limits_{\bw\in\W}\cL(\bw;~\D)\\
    &\leq\ex{}{\hL(\hbw;~S)-\min\limits_{\bw\in\W}\hL(\bw;~S)}+\ex{}{\cL(\hbw;~\D)-\hL(\hbw;~S)}\\
    &\leq \frac{8}{\lambda}\left(\frac{2\,L^2\,d\,\log(1/\delta)}{\eps^2\,n^2}+\frac{L^{\,2}}{n}\right) + \lambda\,M^2
\end{align*}
where the second inequality follows from the fact that $\ex{S\sim\D^n}{\min\limits_{\bw\in\W}\hL(\bw;~S)}\leq \min\limits_{\bw\in\W}\ex{S\sim\D^n}{\hL(\bw;~S)}=\min\limits_{\bw\in\W}\cL(\bw;~\D)$, and the last bound follows from combining (\ref{ineq:gen-error-bound}) with Lemma~\ref{lem:emp_risk_Aobj}. Optimizing this bound in $\lambda$ yields the setting of $\lambda$ in the theorem statement. Plugging that setting of $\lambda$ into the bound yield the stated bound on the excess population loss.

\mypar{A note on the rank assumption} While in this section we presented our result under the assumption that rank of $\grad^2\ell(\bw,z)$ is at most one, one can extend the analysis (by using similar argument in \cite{iyengartowards}) to a rank of $\widetilde{O}\left(\frac{L\sqrt{n+d}}{\beta M}\right)$ without affecting the asymptotic population loss guarantees. In general, to ensure differential privacy to $\mathcal{A}_{\sf ObjP}$, one only needs the following assumption involving the Hessian of individual losses: $\left|{\sf det}\left(\mathbb{I}+\frac{\grad^2\ell(\bw,z)}{\lambda}\right)\right|\leq e^{\epsilon/2}$ for all $z\in\cZ$ and $\bw\in\cW$, rather than a constraint on the rank. 

\subsection{Oracle Efficient Objective Perturbation}
\label{sec:OracEff}

The privacy guarantee of the standard objective perturbation technique is given only when the output is the exact minimizer \cite{CMS, kifer2012private}. In practice, we usually cannot attain the exact minimizer, but rather obtain an approximate minimizer via efficient optimization methods. Hence, in this section we focus on providing a practical version of algorithm $\mathcal{A}_{\sf ObjP}$, called
\emph{approximate objective perturbation} (Algorithm $\mathcal{A}_{\sf ObjP-App}$), that i) is $(\epsilon,\delta)$-differentially private, ii) achieves nearly the same population loss as $\mathcal{A}_{\sf ObjP}$, and iii) only makes $O(n\log n)$ evaluations of the gradient $\grad_\bw\ell(\bw, z)$ at any $\bw\in\cW$ and $z\in\cZ$. The main idea in $\mathcal{A}_{\sf ObjP-App}$ is to first obtain a $\bw_2$ that ensures $\mathcal{J}(\bw_2;S)-\min\limits_\cW \mathcal{J}(\bw;S)$ is at most $\alpha$, and then perturb $\bw_2$ with Gaussian noise to ``fuzz'' the difference between $\bw_2$ and the true minimizer.
In this work, we use Stochastic Variance Reduced Gradient Descent (SVRG) \cite{johnson2013accelerating,xiao2014proximal} as the optimization algorithm. This leads to a construction that requires near linear oracle complexity (i.e., number of gradient evaluations). In particular, $\mathcal{A}_{\sf ObjP-App}$ achieves oracle complexity of $O(n\log n)$ and asymptotically optimal excess population loss.


\begin{algorithm}
	\caption{$\A_{\sf ObjP-App}$: Approximate Objective Perturbation for convex, smooth losses}
	\begin{algorithmic}[1]
		\REQUIRE Private dataset: $S=(z_1, \ldots, z_n)\in \cZ^n$, $L$-Lipschitz, $\beta$-smooth, convex loss function $\ell$, convex set $\cW\subseteq \re^d$, privacy parameters $\epsilon \leq 1,\,  \delta \leq 1/n^2$, regularization parameter $\lambda$, Optimizer $\mathcal{O}:\mathcal{F}\times[0,1]\rightarrow\W$ (where $\mathcal{F}$ is the class of objectives, and the other argument is the optimization accuracy), $\alpha\in[0,1]:$ optimization accuracy.
			\STATE Sample $\bG\sim \cN\left(\bzero, \sigma^2\,\iden_d\right),$ where $\sigma^2= \frac{20\,L^2\,\log(1/\delta)}{\eps^2}$.
			\STATE Let $\mathcal{J}(\bw;S)=\hL\left(\bw;~S\right)+\frac{\langle \bG, ~\bw\rangle}{n}+\lambda\norm{\bw}^2,$ where $\hL(\bw; ~S)\triangleq \frac{1}{n}\sum_{i=1}^n\ell(\bw, ~z_i).$
            {\RETURN $\hbw=\proj_\W\left[ \mathcal{O}\left(\mathcal{J},\alpha\right)+\bH\right]$, where   $\bH\sim\cN\left(\bzero, \sigma^2\,\iden_d\right)$, and $\sigma^2= \frac{40\alpha\,\log(1/\delta)}{\lambda\eps^2}$\label{step:3}}.
	\end{algorithmic}
	\label{Alg:ObjP-smooth-app}
\end{algorithm}

\begin{thm}[Privacy guarantee of $\A_{\sf ObjP-App}$] Suppose that Assumption~\ref{assump:twice-diff} holds and that the smoothness parameter satisfies $\beta \leq \eps\,n\,\lambda$. Then, Algorithm $\A_{\sf ObjP-App}$ is $(\eps, \delta)$-differentially private.
\label{thm:privObjApp}
\end{thm}

\begin{proof}
Let $\bw_1=\arg\min\limits_{\bw\in\W}\underbrace{\hL\left(\bw;~S\right)+\frac{\langle \bG, ~\bw\rangle}{n}+\lambda\norm{\bw}^2}_{\mathcal{J}(\bw,S)}$, and $\bw_2=\mathcal{O}(\mathcal{J},\alpha)$, where $\mathcal{O}$ is the optimizer defined in Algorithm $\A_{\sf ObjP-App}$. Notice that one can compute $\hbw$ from the tuple $(\bw_1,\bw_2-\bw_1+\bH)$ by simple post-processing. Furthermore, the algorithm that outputs $\bw_1$ is $(\epsilon/2,\delta/2)$-differentially private by Theorem \ref{thm:privObjPert}. In the following, we will bound $\|\bw_2-\bw_1\|$ in order to make $(\bw_2-\bw_1 +\bH)$ differentially private, conditioned on the knowledge of $\bw_1$.

As  $\mathcal{J}(\bw,S)$ is $\lambda$-strongly convex, $\mathcal{J}(\bw_2,S)\geq \mathcal{J}(\bw_1,S)+\frac{\lambda}{2}\|\bw_2-\bw_1\|^2$ so that
\begin{align}
    \|\bw_2-\bw_1\|\leq\sqrt{\frac{2\cdot|\mathcal{J}(\bw_2,S)-\mathcal{J}(\bw_1,S)|}{\lambda}}\leq\sqrt{\frac{2\alpha}{\lambda}}.
    \label{eq:abc12}
\end{align}
From eq.~\eqref{eq:abc12} it follows that, conditioned on $\bw_1$,  $\bw_2-\bw_1$ has $\ell_2$-sensitivity of $\sqrt{\frac{8\alpha}{\lambda}}$. Therefore, by the standard analysis of the Gaussian mechanism~\cite{DR14}, it follows that $(\bw_2-\bw_1)+\bH$ (with $\bH$ sampled as in Step \ref{step:3} of Algorithm $\A_{\sf ObjP-App}$) satisfies $(\epsilon/2,\delta/2)$-differential privacy. Therefore by standard composition~\cite{DR14}, the tuple $(\bw_1,\bw_2-\bw_1+\bH)$ (and hence $\hbw$) satisfies $(\epsilon,\delta)$-differential privacy.
\end{proof}

\begin{thm}[Excess population loss guarantee of $\A_{\sf ObjP-App}$]
Let $\D$ be any distribution over $\Z,$ and let $S\sim\D^n$. Suppose that Assumption~\ref{assump:twice-diff} holds and that $\W$ is $M$-bounded. In Algorithm $\A_{\sf ObjP-App}$, set $\lambda= \frac{2\,L}{M}\sqrt{\frac{2}{n}+\frac{4\,d\,\log(1/\delta)}{\eps^2\, n^2}}$, $\alpha=\frac{M^2\lambda}{n^2}$.  Then, we have
\begin{align*}
    \er\left(\A_{\sf ObjP-App};~\D\right)&\leq O\left(M\,L\cdot \max\left(\frac{1}{\sqrt{n}},~ \frac{\sqrt{d\,\log(1/\delta)}}{\eps\, n}\right)\right).
\end{align*}
\label{thm:excessPop}
\end{thm}

\begin{proof}
 Let $\bw_1=\arg\min\limits_{\bw\in\W}{\hL\left(\bw;~S\right)+\frac{\langle \bG, ~\bw\rangle}{n}+\lambda\norm{\bw}^2}$. For $\hbw$ defined in Step \ref{step:3} of $\A_{\sf ObjP-App}$, notice that using Theorem \ref{thm:pop_risk_Aobj},  $$\er\left(\hbw;~\D\right)\leq \er\left(\bw_1;~\D\right)+L\cdot\mathbb{E}\left[\|\hbw-\bw_1\|\right]\leq O\left(M\,L\cdot \max\left(\frac{1}{\sqrt{n}},~ \frac{\sqrt{d\,\log(1/\delta)}}{\eps\, n}\right)\right)+L\cdot\mathbb{E}\left[\|\bH\|\right].$$ Now, $$\mathbb{E}\left[\|\bH\|\right]=O\left(\sqrt{\frac{d\alpha\log(1/\delta)}{\lambda\eps^2}}\right)=O\left(M\,L\cdot\frac{\sqrt{d\log(1/\delta)}}{\epsilon n}\right)$$ when $\alpha=\frac{M^2\lambda}{n^2}$. Therefore, $\er\left(\hbw;~\D\right)\leq O\left(M\,L\cdot \max\left(\frac{1}{\sqrt{n}},~ \frac{\sqrt{d\,\log(1/\delta)}}{\eps\, n}\right)\right)$, which completes the proof.
\end{proof}

\mypar{Oracle complexity} The population loss guarantee of Algorithm $\A_{\sf ObjP-App}$ is independent of the choice of the exact optimizer $\mathcal{O}$, as long it produces a $\hbw\in\W$ for an objective function $\mathcal{J}$ such that\\ $\left[\mathcal{J}(\hbw)-\min\limits_{\bw\in\W}\mathcal{J}(\bw)\right]\leq \alpha$, where $\alpha=\frac{M^2\lambda}{n^2}$ (defined in Theorem \ref{thm:excessPop}). We will now show that if one uses SVRG (Stochastic Variance Reduced Gradient Descent Optimizer) from \cite{johnson2013accelerating,xiao2014proximal,bubeck2015convex} as the optimizer $\mathcal{O}$, then one can achieve an error of at most $\alpha$ using $O\left((n+\beta/\lambda)\log(1/\alpha)\right)$ calls to the gradients of $\ell(\cdot,\cdot)$, for any $\alpha\in(0,1]$. The following theorem immediately gives this. Plugging in the value of $\alpha$ from Theorem \ref{thm:excessPop}, noticing from Theorem \ref{thm:priv_Aobj} that $\beta/\lambda\leq \epsilon n$, and considering $\epsilon, M$ and $L$ to be constants, we get the oracle complexity of Algorithm $\A_{\sf ObjP-App}$ to be $O(n\log(n))$.

\begin{thm}[Convergence of SVRG \cite{johnson2013accelerating,xiao2014proximal,bubeck2015convex}]
Let $f_1,\cdots,f_n$ be $\beta$-smooth, $\lambda$-strongly convex functions over $\W$, $\mathcal{F}(\bw)=\frac{1}{n}\sum\limits_{i=1}^n f_i(\bw)$, and $\bw^* \triangleq \arg\min_{\bw\in \W} \mathcal{F}(\bw)$.  Let $\by^{(1)}\in\W$ be an arbitrary initial point. For $t=\{1,2,\cdots\}$, let $\bw^{(t)}_1=\by^{(t)}$. For $s\in[k]$, let
$$\bw^{(t)}_{s+1}=\proj_\W\left[\bw^{(t)}_{s}-\frac{1}{10\beta}\left(\grad f_{i^{(t)}_s}\left(\bw^{(t)}_{s}\right)-\grad f_{i^{(t)}_s}\left(\by^{(t)}\right)+\grad \mathcal{F}\left(\by^{(t)}\right)\right)\right],$$
where $i^{(t)}_s$ is drawn uniformly at random from $[n]$, and $\by^{(t+1)}=\frac{1}{k}\sum\limits_{s=1}^k \bw^{(t)}_s$. Then, for $k=20\beta/\lambda$ it holds that:
$$\mathbb{E}\left[\mathcal{F}\left(\by^{(t+1)}\right)\right]-\mathcal{F}\left(\bw^*\right)\leq 0.9^t\left(\mathcal{F}\left(\by^{(1)}\right)-\mathcal{F}\left(\bw^*\right)\right).$$
\label{thm:SVRG}
\end{thm}


\else

\fi

\subsection*{Acknowledgements}
We thank Adam Smith, Thomas Steinke and Jon Ullman for the insightful discussions of the problem at the early stages of this project. We are also grateful to Tomer Koren for bringing the Moreau-Yosida smoothing technique to our attention.

\newpage
\bibliographystyle{alpha}
\bibliography{references}

\newcommand{\etalchar}[1]{$^{#1}$}
\begin{thebibliography}{DKM{\etalchar{+}}06}

\bibitem[ACG{\etalchar{+}}16]{abadi2016deep}
Martin Abadi, Andy Chu, Ian Goodfellow, H~Brendan McMahan, Ilya Mironov, Kunal
  Talwar, and Li~Zhang.
\newblock Deep learning with differential privacy.
\newblock In {\em Proceedings of the 2016 ACM SIGSAC Conference on Computer and
  Communications Security}, pages 308--318. ACM, 2016.

\bibitem[BE02]{bousquet2002stability}
Olivier Bousquet and Andr{\'e} Elisseeff.
\newblock Stability and generalization.
\newblock {\em Journal of machine learning research}, 2(Mar):499--526, 2002.

\bibitem[BNS{\etalchar{+}}16]{bassily2016algorithmic}
Raef Bassily, Kobbi Nissim, Adam Smith, Thomas Steinke, Uri Stemmer, and
  Jonathan Ullman.
\newblock Algorithmic stability for adaptive data analysis.
\newblock In {\em Proceedings of the forty-eighth annual ACM symposium on
  Theory of Computing}, pages 1046--1059. ACM, 2016.

\bibitem[BST14]{bassily2014differentially}
Raef Bassily, Adam Smith, and Abhradeep Thakurta.
\newblock Differentially private empirical risk minimization: Efficient
  algorithms and tight error bounds.
\newblock {\em arXiv preprint arXiv:1405.7085}, 2014.

\bibitem[Bub15]{bubeck2015convex}
S{\'e}bastien Bubeck.
\newblock Convex optimization: Algorithms and complexity.
\newblock {\em Foundations and Trends{\textregistered} in Machine Learning},
  8(3-4):231--357, 2015.

\bibitem[Can11]{candes_opt_notes}
Emmanuel Candes.
\newblock {\em Mathematical optimization}, volume Lec. notes: MATH 301.
\newblock Stanford Univesity, 2011.

\bibitem[CM08]{CM08}
Kamalika Chaudhuri and Claire Monteleoni.
\newblock Privacy-preserving logistic regression.
\newblock In Daphne Koller, Dale Schuurmans, Yoshua Bengio, and L{\'e}on
  Bottou, editors, {\em NIPS}. MIT Press, 2008.

\bibitem[CMS11]{CMS}
Kamalika Chaudhuri, Claire Monteleoni, and Anand~D Sarwate.
\newblock Differentially private empirical risk minimization.
\newblock {\em Journal of Machine Learning Research}, 12(Mar):1069--1109, 2011.

\bibitem[DFH{\etalchar{+}}15]{dwork2015preserving}
Cynthia Dwork, Vitaly Feldman, Moritz Hardt, Toniann Pitassi, Omer Reingold,
  and Aaron~Leon Roth.
\newblock Preserving statistical validity in adaptive data analysis.
\newblock In {\em Proceedings of the forty-seventh annual ACM symposium on
  Theory of computing}, pages 117--126. ACM, 2015.

\bibitem[DJW13]{DuchiJW13}
John~C. Duchi, Michael~I. Jordan, and Martin~J. Wainwright.
\newblock Local privacy and statistical minimax rates.
\newblock In {\em IEEE 54th Annual Symposium on Foundations of Computer Science
  (FOCS)}, pages 429--438, 2013.

\bibitem[DKM{\etalchar{+}}06]{DKMMN06}
Cynthia Dwork, Krishnaram Kenthapadi, Frank McSherry, Ilya Mironov, and Moni
  Naor.
\newblock Our data, ourselves: Privacy via distributed noise generation.
\newblock In {\em EUROCRYPT}, 2006.

\bibitem[DMNS06]{DMNS06}
Cynthia Dwork, Frank McSherry, Kobbi Nissim, and Adam Smith.
\newblock Calibrating noise to sensitivity in private data analysis.
\newblock In {\em Theory of Cryptography Conference}, pages 265--284. Springer,
  2006.

\bibitem[DR14a]{dwork2013algorithmic}
Cynthia Dwork and Aaron Roth.
\newblock The algorithmic foundations of differential privacy.
\newblock {\em Foundations and Trends{\textregistered} in Theoretical Computer
  Science}, 9(3--4):211--407, 2014.

\bibitem[DR{\etalchar{+}}14b]{DR14}
Cynthia Dwork, Aaron Roth, et~al.
\newblock The algorithmic foundations of differential privacy.
\newblock {\em Foundations and Trends in Theoretical Computer Science},
  9(3-4):211--407, 2014.

\bibitem[Fel16]{feldman2016generalization}
Vitaly Feldman.
\newblock Generalization of erm in stochastic convex optimization: The
  dimension strikes back.
\newblock In {\em Advances in Neural Information Processing Systems}, pages
  3576--3584, 2016.

\bibitem[FGV15]{FeldmanGV:15}
Vitaly Feldman, Cristobal Guzman, and Santosh Vempala.
\newblock Statistical query algorithms for mean vector estimation and
  stochastic convex optimization.
\newblock {\em CoRR}, abs/1512.09170, 2015.
\newblock Extended abstract in SODA 2017.

\bibitem[FV19]{feldman2019high}
Vitaly Feldman and Jan Vondrak.
\newblock High probability generalization bounds for uniformly stable
  algorithms with nearly optimal rate.
\newblock {\em arXiv preprint arXiv:1902.10710}, 2019.

\bibitem[HRS15]{hardt2015train}
Moritz Hardt, Benjamin Recht, and Yoram Singer.
\newblock Train faster, generalize better: Stability of stochastic gradient
  descent.
\newblock {\em arXiv preprint arXiv:1509.01240}, 2015.

\bibitem[INS{\etalchar{+}}19]{iyengartowards}
Roger Iyengar, Joseph~P Near, Dawn Song, Om~Thakkar, Abhradeep Thakurta, and
  Lun Wang.
\newblock Towards practical differentially private convex optimization.
\newblock In {\em IEEE S and P (Oakland)}, 2019.

\bibitem[JKT12]{jain2012differentially}
Prateek Jain, Pravesh Kothari, and Abhradeep Thakurta.
\newblock Differentially private online learning.
\newblock In {\em 25th Annual Conference on Learning Theory (COLT)}, pages
  24.1--24.34, 2012.

\bibitem[JT14]{JTOpt13}
Prateek Jain and Abhradeep Thakurta.
\newblock (near) dimension independent risk bounds for differentially private
  learning.
\newblock In {\em ICML}, 2014.

\bibitem[JZ13]{johnson2013accelerating}
Rie Johnson and Tong Zhang.
\newblock Accelerating stochastic gradient descent using predictive variance
  reduction.
\newblock In {\em Advances in neural information processing systems}, pages
  315--323, 2013.

\bibitem[KST12]{kifer2012private}
Daniel Kifer, Adam Smith, and Abhradeep Thakurta.
\newblock Private convex empirical risk minimization and high-dimensional
  regression.
\newblock In {\em Conference on Learning Theory}, pages 25--1, 2012.

\bibitem[NB10]{nedic2010effect}
Angelia Nedi{\'c} and Dimitri~P Bertsekas.
\newblock The effect of deterministic noise in subgradient methods.
\newblock {\em Mathematical programming}, 125(1):75--99, 2010.

\bibitem[Nes05]{nesterov2005smooth}
Yu~Nesterov.
\newblock Smooth minimization of non-smooth functions.
\newblock {\em Mathematical programming}, 103(1):127--152, 2005.

\bibitem[SCS13]{song2013stochastic}
Shuang Song, Kamalika Chaudhuri, and Anand~D Sarwate.
\newblock Stochastic gradient descent with differentially private updates.
\newblock In {\em IEEE Global Conference on Signal and Information Processing},
  2013.

\bibitem[SSBD14]{shalev2014understanding}
Shai Shalev-Shwartz and Shai Ben-David.
\newblock {\em Understanding machine learning: From theory to algorithms}.
\newblock Cambridge university press, 2014.

\bibitem[SSSSS09]{SSSS}
Shai Shalev-Shwartz, Ohad Shamir, Nathan Srebro, and Karthik Sridharan.
\newblock {Stochastic Convex Optimization}.
\newblock In {\em COLT}, 2009.

\bibitem[ST13]{ST13sparse}
Adam Smith and Abhradeep Thakurta.
\newblock Differentially private feature selection via stability arguments, and
  the robustness of the {LASSO}.
\newblock In {\em Conference on Learning Theory (COLT)}, pages 819--850, 2013.

\bibitem[STU17]{smith2017interaction}
Adam Smith, Abhradeep Thakurta, and Jalaj Upadhyay.
\newblock Is interaction necessary for distributed private learning?
\newblock In {\em IEEE Security \& Privacy}, pages 58--77, 2017.

\bibitem[TTZ15]{talwar2015nearly}
Kunal Talwar, Abhradeep Thakurta, and Li~Zhang.
\newblock Nearly optimal private {LASSO}.
\newblock In {\em Proceedings of the 28th International Conference on Neural
  Information Processing Systems}, volume~2, pages 3025--3033, 2015.

\bibitem[Ull15]{ullman2015private}
Jonathan Ullman.
\newblock Private multiplicative weights beyond linear queries.
\newblock In {\em Proceedings of the 34th ACM SIGMOD-SIGACT-SIGAI Symposium on
  Principles of Database Systems}, pages 303--312. ACM, 2015.

\bibitem[WLK{\etalchar{+}}17]{wu2017bolt}
Xi~Wu, Fengan Li, Arun Kumar, Kamalika Chaudhuri, Somesh Jha, and Jeffrey
  Naughton.
\newblock Bolt-on differential privacy for scalable stochastic gradient
  descent-based analytics.
\newblock In {\em SIGMOD}. ACM, 2017.

\bibitem[WYX17]{wang2017differentially}
Di~Wang, Minwei Ye, and Jinhui Xu.
\newblock Differentially private empirical risk minimization revisited: Faster
  and more general.
\newblock In {\em Advances in Neural Information Processing Systems}, pages
  2722--2731, 2017.

\bibitem[XZ14]{xiao2014proximal}
Lin Xiao and Tong Zhang.
\newblock A proximal stochastic gradient method with progressive variance
  reduction.
\newblock {\em SIAM Journal on Optimization}, 24(4):2057--2075, 2014.

\end{thebibliography}

\ifnum\nips=0
\appendix
\section{Proof of Lemma~\ref{lem:stability}}\label{app:A}

Consider $T$ iterations of $\Ansgd$. Let $\bG_1, \ldots, \bG_T$ denote the noise vectors and $\cI_1, \ldots, \cI_T \in [n]^m$ denote the \emph{index} sets of the mini-batches selected in the $T$ iterations. Consider any pair of datasets $S=(z_1, \ldots, z_k, \ldots, z_n)$ and $S'=(z_1, \ldots, z'_k, \ldots, z_n)$ differing in exactly one data point $z_k\neq z'_k$ for some fixed $k\in [n]$. Let $\bw_0, \bw_1, \ldots, \bw_T$ and $\bw_0, \bw'_1, \ldots, \bw'_T$ denote the trajectories of $\Ansgd$ corresponding to input datasets $S$ and $S'$, respectively. For any $t\in [T],$ let $\xi_t\triangleq \bw_t-\bw'_t$.

We follow the proof technique of \cite[Lemma~4.3]{feldman2019high}. We prove the following claim via induction on $t$:
$$\ex{}{\norm{\xi_t}}\leq 2\,L\,\frac{\eta\, t}{n},$$
where the expectation is taken over $\cI_0,\ldots,\cI_{t-1}, \bG_0,\ldots,\bG_{t-1}$. First, it's trivial to see that the claim is true for $t=0$. Suppose the claim holds for all $t\leq \tau$. Fix the randomness in  $\bG_{\tau}$ and $\cI_{\tau}$. Let $r$ denote the number of occurrences of the index $k$ (where $S$ and $S'$ differ) in $\cI_{\tau}$. By the non-expansiveness property of the gradient update step, we have
\begin{align*}
\norm{\xi_{\tau+1}}&\leq \norm{\xi_{\tau}}+2\,L\,\eta\,\frac{r}{m} \label{ineq:recurs}
\end{align*}
Now, we now invoke the randomness in $\bG_{\tau}$ and $\cI_{\tau}$. Note that $r$ is a Binomial random variable with mean $m/n$. Hence, by taking expectation and using the induction hypothesis, we end up with
\begin{align*}
\ex{\substack{\cI_0,\ldots,\cI_{\tau}\\\bG_0,\ldots,\bG_{\tau}}}{\norm{\xi_{\tau+1}}}&\leq 2\,L\,\frac{\eta\,(\tau+1)}{n}
\end{align*}
This proves the claim. Now, let $\brw_T=\frac{1}{T}\sum_{t=1}^T\bw_t$ and $\brw'_T=\frac{1}{T}\sum_{t=1}^T\bw'_T$. Since $\ell$ is $L$-Lipschitz, thus for every $z\in\Z$, we have
\begin{align*}
    \ex{\substack{\cI_0,\ldots,\cI_{t-1}\\\bG_0,\ldots,\bG_{t-1}}}{\ell(\brw_T, ~z)-\ell(\brw'_T, ~z)}&\leq L \ex{\substack{\cI_0,\ldots,\cI_{t-1}\\\bG_0,\ldots,\bG_{t-1}}}{\norm{\brw_T-\brw'_T}}\leq L \frac{1}{T}\sum_{t=1}^T \ex{\cI_t, \bG_t}{\norm{\xi_t}}\hspace{1cm}\\
    &\leq 2 L^2\,\frac{\eta}{n\,T}\frac{T(T+1)}{2}=L^2\, \frac{\eta\,(T+1)}{n}
\end{align*}
This completes the proof.

\section{Proof of Lipschitz property of Moreau envelope (Lemma~\ref{lem:prop_moreau})}\label{app:B}
Fix any $\bw\in \W$. We will show that $\norm{\nabla f_{\beta}(\bw)}\leq 2L.$ Define $g(\bv)\triangleq f(\bv)+\frac{\beta}{2}\norm{\bv-\bw}^2,~ \bv\in\W$. Note that $\prox_{f/\beta}(\bw)=\arg\min\limits_{\bv\in\W}g(\bv).$ Let $\bv^*$ denote $\prox_{f/\beta}(\bw)$. Now, observe that
\begin{align*}
    0&\leq g(\bw)-g(\bv^*)=f(\bw)-f(\bv^*)-\frac{\beta}{2}\norm{\bw-\bv^*}^2
\end{align*}
Thus, we have
\begin{align*}
    \frac{\beta}{2}\norm{\bw-\bv^*}^2&\leq f(\bw)-f(\bv^*)\leq L\,\norm{\bw-\bv^*}
\end{align*}
where the last inequality follows from the fact that $f$ is $L$-Lipschitz. Thus, we get $\norm{\bw-\bv^*}\leq 2\,L/\beta.$ By property~\ref{prop:moreau_3}, we have $\norm{\nabla f_{\beta}(\bw)}=\beta\,\norm{\bw-\bv^*}$. This together with the above bound gives the desired result.

\section{Optimality of Our Bounds}\label{sec:lower} 

Our upper bounds in Sections~\ref{sec:smooth} and \ref{sec:non-smooth} are tight (up to logarithmic factors in $1/\delta$). In particular, our bounds match a lower bound of  $\Omega\left(M\,L\cdot \max\left(\frac{1}{\sqrt{n}},~\frac{\sqrt{d}}{n}\right)\right)$ on the excess population loss. The first term is simply the known lower bound on the excess population loss in the non-private setting. The second  term follows from the lower bound in \cite{bassily2014differentially} on excess empirical loss, and the fact that a lower bound on excess empirical loss implies nearly the same lower bound on the excess population loss. We elaborate on this below.  

\paragraph{Reduction from Private ERM to Private SCO:} For any $\gamma>0$, suppose there is $\left(\frac{\eps}{4\,\log(2/\delta)}, \frac{e^{-\eps}\delta}{8\,\log(2/\delta)}\right)$-differentially private algorithm $\A$ such that for any distribution on a domain $\Z$, when $\A$ is given a sample $T\sim\D^n,$ it yields expected excess population loss $\er(\A;~\D) \leq \gamma$. Then, there is $(\eps, \delta)$-differentially private algorithm $\cB$ that when given any dataset $S\in\Z^n$, it yields expected excess empirical loss $\her(\cB; ~S)\triangleq \ex{\cB}{\hL\left(\cB(S); S\right)}-\min\limits_{\bw}\hL(\bw; S)\leq \gamma$.

Fix any $\gamma>0$. Suppose algorithm $\A$ described above exists. We construct algorithm $\cB$ as follows:
\begin{enumerate}
    \item Given input dataset $S\in\Z^n,$ let $\D_S$ be the empirical distribution induced by $S$.
    \item Sample $T\sim\D_S^n$. \label{step:unif-samp}
    \item Return $\A(T)$
\end{enumerate}
First, note that $\her(\cB; S)\leq \gamma$. This easily follows from the fact that for any $\bw$, $\cL(\bw; \D_S)=\hL(\bw; S)$. In particular, observe that
\begin{align*}
&\ex{\cB}{\hL\left(\cB(S); S\right)}-\min\limits_{\bw}\hL(\bw; S)=\ex{T\sim\D_S^n, \A}{\cL\left(\A(T); ~\D_S\right)}-\min\limits_{\bw}\cL(\bw;~\D_S)\\
&=\er\left(\A; \D_S\right)\leq \gamma.
\end{align*}
Next, we show that $\cB$ is $(\eps, \delta)$-differentially private. Let $S=(z_1, \ldots, z_k, \ldots, z_n), S'=(z_1, \ldots, z'_k, \ldots, z_n)$ be neighboring datasets differing in single point whose index is $k\in [n]$. Let $T, T'$ be the samples obtained by running $\cB$ on $S, S',$ respectively, with the same set of random coins in Step~\ref{step:unif-samp}. More precisely, let $R$ denote the random sampling procedure used in Step~\ref{step:unif-samp}, and define $T=R(S)$ and $T'=R(S')$. Let $r$ be the number of times the $k$-th point of the input dataset is sampled by $R$. Hence, $r=\lvert T\Delta T'\rvert$, i.e., $r$ is the number of points where $T$ and $T'$ differ. By Chernoff's bound, $r\leq 4\,\log(2/\delta)$ with probability $1-\delta/2$. Let $\V$ be any measurable subset of the range of $\cB$. Observe that
\begin{align*}
    \pr{\cB}{\cB(S)\in\V}&=\pr{\A, R}{\A(T)\in\V}\\
    &\leq \pr{\A, R}{\A(T)\in\V\vert ~r\leq 4\,\log(2/\delta)}\cdot\pr{}{~r\leq 4\,\log(2/\delta)} +\delta/2\\
    &\leq e^{\frac{r\,\eps}{4\,\log(2/\delta)}}\cdot \pr{\A, R}{\A(T')\in\V\vert~ r\leq 4\,\log(2/\delta)}\cdot\pr{}{~r\leq 4\,\log(2/\delta)} +\frac{\delta}{2}+r\,e^{\frac{r\,\eps}{4\,\log(2/\delta)}}\frac{e^{-\eps}\delta}{8\,\log(2/\delta)}\\
    &\leq e^{\eps}\cdot \pr{\A, R}{\A(T')\in\V} +\delta\\
    &=e^{\eps}\cdot\pr{\cB}{\cB(S')\in\V}+\delta,
\end{align*}
where the third inequality follows from the fact that $\A$ is $\left(\frac{\eps}{4\,\log(2/\delta)}, \frac{\delta}{2}\right)$-differentially private and group differential privacy (e.g.~\cite{dwork2013algorithmic}). This shows that $\cB$ is $(\eps, \delta)$-differentially private, proving the reduction, and hence, the lower bound. 
\fi

\end{document}